\documentclass{article}

\usepackage{microtype}
\usepackage{graphicx}
\usepackage{subfigure}
\usepackage{booktabs} 

\usepackage{hyperref}

\usepackage{amsmath,nccmath}
\usepackage{amssymb}
\usepackage{mathtools}
\usepackage{amsthm}

\theoremstyle{plain}
\newtheorem{theorem}{Theorem}[section]
\newtheorem{proposition}[theorem]{Proposition}

\theoremstyle{definition}
\newtheorem{definition}[theorem]{Definition}

\newtheorem{example}[theorem]{Example}
\theoremstyle{remark}
\newtheorem{remark}[theorem]{Remark}
\usepackage{natbib}

\usepackage[textsize=tiny]{todonotes}

\usepackage{enumitem}
\usepackage{multirow}
\usepackage{setspace}
\usepackage{graphicx}
\usepackage{subfigure} 

\def\mcl#1{\mathcal{#1}}
\def\bracket#1{\left\langle #1\right\rangle}

\def\nn{\nonumber}

\def\mr{\mathrm}

\def\alg{\mcl{A}}
\def\modu{\mcl{M}}

\def\Bbracket#1{\bigg\langle #1\bigg\rangle}

\def\red#1{\textcolor{black}{#1}}
\def\r#1{\mathbb{R}^{#1}}

\DeclareSymbolFont{EulerExtension}{U}{euex}{m}{n}
\DeclareMathSymbol{\euintop}{\mathop} {EulerExtension}{"52}
\DeclareMathSymbol{\euointop}{\mathop} {EulerExtension}{"48}

\allowdisplaybreaks

\newenvironment{mythm}[1][]{\medskip\par\noindent{\bfseries #1}\ \,\,\em}{\medskip\par}

\author{Yuka Hashimoto$^{1,2}$\quad Masahiro Ikeda$^{2,3}$\quad Hachem Kadri$^4$\medskip\\
{\normalsize 1. NTT Corporation, Tokyo, Japan}\\
{\normalsize 2. Center for Advanced Intelligence Project, RIKEN, Tokyo, Japan}\\
{\normalsize 3. Keio University, Yokohama, Japan}\\
{\normalsize 4. Aix-Marseille University, CNRS, LIS, Marseille, France}}

\date{}

\title{$C^*$-Algebraic Machine Learning: Moving in a New Direction}

\begin{document}

\twocolumn[
\maketitle




]




\begin{abstract}

Machine learning has a long collaborative tradition with several fields of mathematics, such as statistics, probability and linear algebra. We propose a new direction for machine learning research: {\em $C^*$-algebraic ML}---a cross-fertilization between $C^*$-algebra and machine learning. The mathematical concept of $C^*$-algebra is a natural generalization of the space of complex numbers. It enables us to unify existing learning strategies, and construct a new framework for more diverse and information-rich data models. We explain why and how to use $C^*$-algebras in machine learning, and provide technical considerations that go into the design of $C^*$-algebraic learning models in the contexts of kernel methods and neural networks. Furthermore, we discuss open questions and challenges in $C^*$-algebraic ML and give our thoughts for future development and applications.
\end{abstract}

\section{Introduction}
\label{introduction}

Machine learning problems and methods are currently becoming more and more complicated.
We have many types of structured data, such as time-series data, image data, and graph data.
In addition, not only are the models large, but multiple models and tasks have to be considered in some situations.

To address these situations, we propose {\em $C^*$-algebraic machine learning}: application of $C^*$-algebra to machine learning methods.
Typical examples of $C^*$-algebras are the space of continuous functions on a compact space and the space of bounded linear operators on a Hilbert space.
$C^*$-algebra was first proposed in quantum mechanics to model physical observables and
has been investigated in pure mathematics, mathematical physics, and quantum mechanics.
Whereas its rich mathematical and theoretical investigations, its main application is limited to quantum mechanics.
In the current situation in machine learning, we believe that it is time to apply these rich investigations to machine learning methods.
Since $C^*$-algebras enable us to unify complex values, matrices, functions, and linear operators,
we expect that the generalization of machine learning methods using $C^*$-algebras allows us to unify existing methods and construct a framework 
for more complicated data and models.
Figure~\ref{fig:overview} shows an overview of the $C^*$-algebraic machine learning.

In this paper, we mainly focus on two approaches: kernel methods and neural networks.
For kernel methods, most of existing methods are realized using reproducing kernel Hilbert spaces~(RKHSs) or vector-valued RKHS~(vvRKHS), which are constructed by positive definite kernels~\citep{scholkopf01,saitoh16}.
The reproducing property enables us to evaluate the value of a function at a point using the inner product, which makes it easy for us to implement algorithms and analyze them theoretically.
Moreover, we can apply kernel methods to probabilistic and statistical settings by embedding probability measures in an RKHS. This embedding is called the kernel mean embedding.
However, since RKHSs~(resp. vvRKHSs) are complex- (resp. vector-) valued function spaces, the output of the models is usually complex- or vector-valued.
In addition, appropriate ways of the construction of positive definite kernels are not trivial.
The generalization of RKHS by means of the $C^*$-algebra enables us to output more general data, such as functions and operators~\citep{hashimoto21}.
Moreover, $C^*$-algebras give us a method to construct $C^*$-algebra-valued positive definite kernels for structured data~\citep{hashimoto23-aistats}.
The noncommutative product structure in $C^*$-algebras ($ab\neq ba$ for elements $a,b$ in the $C^*$-algebra) enables us to construct an operation that goes beyond the multiplication and convolution. 

As for neural networks, the models are becoming larger and more complicated.
For example, in ensemble learning~\citep{dong20,Ganaie21}, multitask learning~\citep{zhang14,ruder19}, and meta-learning~\citep{ravi17,finn17,rusu19}, we need to consider multiple tasks and models.
In addition, large language models (LLMs) have large numbers of learning parameters and need large numbers of training samples.
To fully extract features of data in these architectures, we expect that $C^*$-algebras play an important role since they enable us to represent multiple models and tasks simultaneously~\citep{hashimoto22}.
In addition, although one reason for the success of current large modes like LLMs is a large number of training samples, in some applications, we do not have enough data to train models.
For example, we do not always have enough healthcare data, and for anomaly detection, we do not always have enough abnormal data.
Moreover, federated learning has been investigated to analyze privacy data distributed in multiple nodes without sharing it with other nodes~\citep{mcmahan17,bonawitz21}.
In these cases, we cannot rely on the powerfulness of neural network models coming from large numbers of training samples.
We believe that the rich structure of models with $C^*$-algebras will offer a new approach to address these situations.

In this paper, we discuss known advantages of applying $C^*$-algebra to machine learning, as summarized as follows:
\begin{itemize}[leftmargin=*,nosep]
\item We can generalize the complex-valued inner product to a $C^*$-algebra-valued inner product.
Since many machine learning methods involve the inner product, such as projection and computing correlations, the generalization can help us extract data features
effectively in these methods.
\item Using the $C^*$-algebra-valued inner product, we can generalize RKHS by means of $C^*$-algebra and learn function- and operator-valued maps (Subsection~\ref{subsec:rkhm}).
\item We can design positive definite kernels for structured data using the noncommutative product (Subsection~\ref{subsec:rkhm}).
\item We can use the norm of the $C^*$-algebra to alleviate the dependency of generalization error bound on the output dimension (Subsection~\ref{subsec:rkhm}).
\item Using the generalization of kernel mean embedding by means of $C^*$-algebra, we can analyze operator-valued measures such as positive operator-valued measures and spectral measures (Subsection~\ref{subsec:kme}).
\item We can continuously combine multiple models and use the tools for functions, which can be applied to ensemble, multitask, and meta-learning (Subsection~\ref{subsec:c_star_net}).
\item The noncommutative product structures in $C^*$-algebras induce interactions among models (Subsection~\ref{subsec:interaction}).
\item We can construct group equivariant 
neural networks using the products in group $C^*$-algebras (Subsection~\ref{subsec:interaction}).
\end{itemize}
In each section and subsection, we discuss the above advantages in more detail.
Then, we go into the technical details to show that the notion of $C^*$-algebra can be adapted 
to machine learning methods such as kernel methods and \red{neural networks}.
%
%
We also provide new results showing the advantages 
of applying $C^*$-algebra. Specifically,
\begin{itemize}[leftmargin=*,nosep]
 \item By generalizing neural networks by means $C^*$-algebra, we can show that even if the activation functions are linear, the expressiveness of the generalized network, called $C^*$-algebra net, grows as the depth grows~(Subsection~\ref{subsec:expressiveness}).
\item $C^*$-algebra net 
fills the gap between convex and nonconvex optimization for neural networks (Subsection~\ref{subsec:optimization}).
\end{itemize}

Finally, we discuss future directions of $C^*$-algebraic machine learning.
We discuss open problems and several possible examples of $C^*$-algebras that could be 
useful in machine learning problems.

\begin{figure}[t]
    \centering
    \includegraphics[scale=0.45]{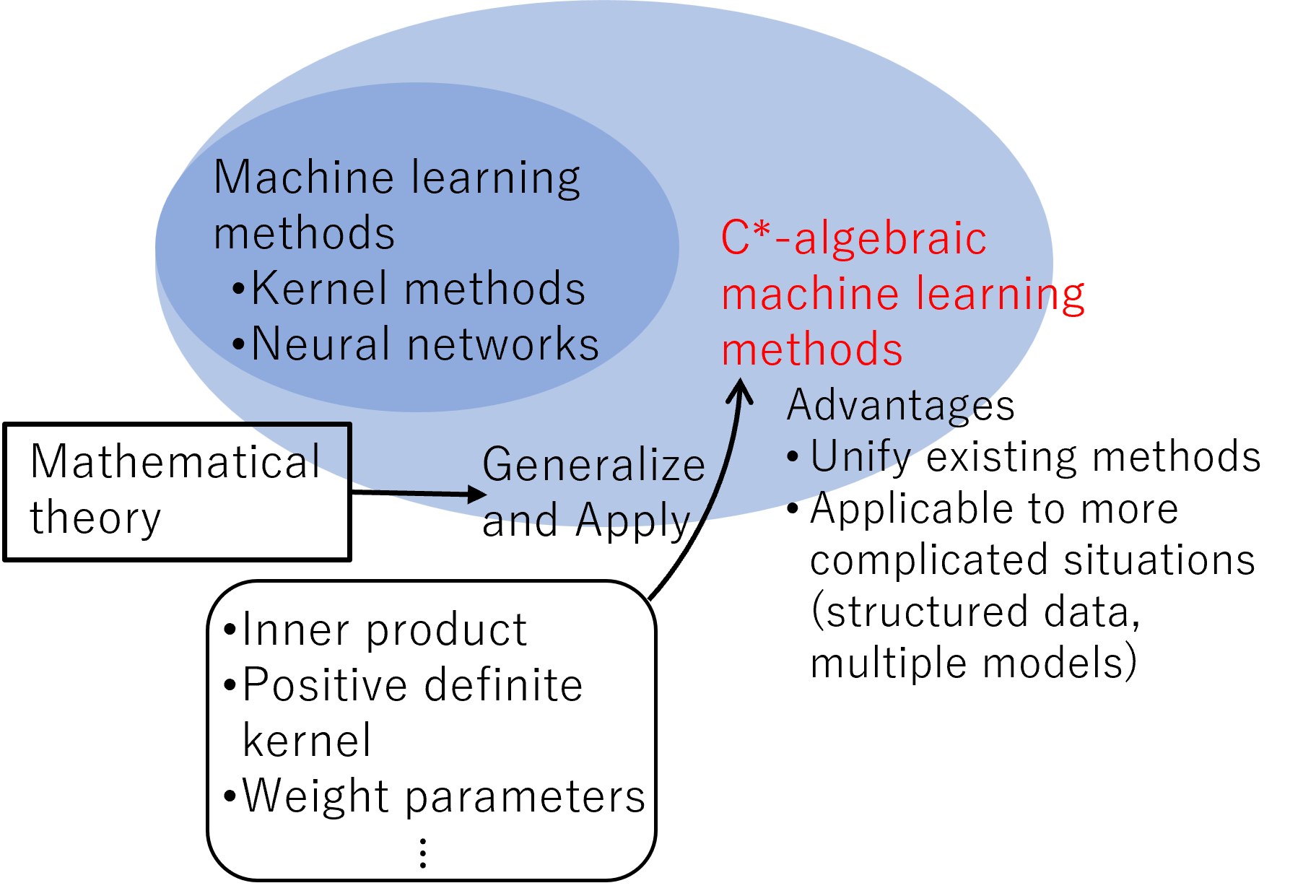}\vspace{-.3cm}
    \caption{Overview of the $C^*$-algebraic machine learning}
    \label{fig:overview}\vspace{-.3cm}
\end{figure}

\section{$C^*$-algebra}\label{sec:c_alg}
$C^*$-algebra is a natural generalization of the space of complex numbers.
Thus, we can naturally generalize real- or complex-valued notions necessary to construct machine learning algorithms.
In $C^*$-algebras, we have arithmetic operations, addition, subtraction, and multiplication, which are fundamental for arbitrary algorithms.
Moreover, we have the structure of involution, which is a generalization of the complex conjugate and is necessary to generalize complex-valued notions.
For example, positive definite kernels for defining RKHSs are, in general, defined as complex-valued functions.
Other important notions for machine learning algorithms are the magnitude and the order.
For example, to evaluate the discrepancy between two values, we consider the magnitude of the difference between the two values.
In addition, we consider minimization or maximization problems in many situations.
The notions of minimum and maximum are based on an order.
We can generalize the order for real values to that in $C^*$-algebras.
In the following, we will see how we can technically generalize real- or complex-valued notions to  $C^*$-algebra-valued ones.
We first recall the definition of $C^*$-algebra.
\begin{definition}[$C^*$-algebra]~\label{def:c*_algebra}
A set $\mcl{A}$ is called a {\em $C^*$-algebra} if it satisfies the following three conditions:
\begin{enumerate}[leftmargin=*,topsep=-3pt]
 \item $\mcl{A}$ is an algebra over $\mathbb{C}$ and equipped with a bijection $(\cdot)^*:\mcl{A}\to\mcl{A}$ that satisfies for $\alpha,\beta\in\mathbb{C}$ and $c,d\in\alg$,\smallskip\\
 $\bullet$ $(\alpha c+\beta d)^*=\overline{\alpha}c^*+\overline{\beta}d^*$,\\
 $\bullet$ $(cd)^*=d^*c^*$,\qquad
 $\bullet$ $(c^*)^*=c$.
 \item $\mcl{A}$ is a Banach space equipped with the norm $\Vert\cdot\Vert_{\alg}$, and for $c,d\in\mcl{A}$, $\Vert cd\Vert_{\alg}\le\Vert c\Vert_{\alg}\Vert d\Vert_{\alg}$ holds.

 \item For $c\in\mcl{A}$, $\Vert c^*c\Vert_{\alg}=\Vert c\Vert_{\alg}^2$ holds. ($C^*$-property)
\end{enumerate}
\end{definition}

The involution $(\cdot)^*$ is a generalization of the complex conjugate.
We can define an $\alg$-valued absolute value to evaluate the magnitude of elements in $\alg$ and compare the absolute value of elements with the partial order defined as follows.
\begin{definition}[Positive]~\label{def:positive}
An element $c\in\mcl{A}$ is called {\em positive} if there exists $d\in\mcl{A}$ such that $c=d^*d$.
We denote $c\le_{\alg} d$ if $d-c$ is positive for $c,d\in\alg$.
We denote by $\alg_+$ the subset of $\alg$ composed of all positive elements in $\alg$.
\end{definition}
For $c\in\alg$, the $\alg$-valued absolute value $\vert c\vert_{\alg}$ is defined as a unique element $d\in\alg_+$ that satisfies $d^2= c^*c$.
Using the above notions, we can define an $\alg_+$-valued objective function and optimize it in the sense of the order $\le_{\alg}$.

We list typical examples of $C^*$-algebras that are useful for machine learning below.
We can find more theoretical properties and examples in, for example, \citet{murphy90} and \citet{davidson96}.

\begin{example} 
\label{ex:vonNeumann}

~\vspace{-0.4cm}

\begin{enumerate}[leftmargin=*]
\item Let $\alg=C(\mcl{Z})$, the $C^*$-algebra of continuous functions on a compact Hausdorff space $\mcl{Z}$.
The product of two functions $c,d\in\alg$ is defined as $(cd)(z)=c(z)d(z)$ for $z\in\mcl{Z}$, the involution is defined as $c^*(z)=\overline{c(z)}$, the norm is the supnorm.
An element $c\in\alg$ is positive if and only if $c(z)\ge 0$ for any $z\in\mcl{Z}$.
Note that the product is commutative, i.e., $cd=dc$ for $c,d\in\alg$.
For example, we can use $C(\mcl{Z})$ for representing functional data and combining multiple models continuously (see Subsection~\ref{subsec:c_star_net}).

\item Let $\alg=\mcl{B}(\mcl{W})$, the $C^*$-algebra of bounded linear operators on a Hilbert space $\mcl{W}$.
The product is the product~(the composition) of operators, the involution is the adjoint,
and the norm is the operator norm.
An element $c\in\alg$ is positive if and only if $c$ is Hermitian positive semi-definite.
Note that, unlike the first example, the product is noncommutative, i.e., $cd\neq dc$ for $c,d\in\alg$.
For example, we can use $\mcl{B}(\mcl{W})$ for treating spectral and positive operator-valued measures~(see Subsection~\ref{subsec:kme}).

\item If $\mcl{W}$ is a $d$-dimensional space, then $\mcl{B}(\mcl{W})$ is the $C^*$-algebra of squared matrices $\mathbb{C}^{d\times d}$.
The space of block diagonal squared matrix is a $C^*$-subalgebra of $\mathbb{C}^{d\times d}$.
For example, we can use $\mathbb{C}^{d\times d}$ to represent adjoint matrices of graphs and images.

\item 
The group $C^*$-algebra on a finite discrete group $G$, which is denoted as $C^*(G)$, is the set of maps from $G$ to $\mathbb{C}$.
The product is defined as $(a\cdot b)(g)=\sum_{h\in G}a(h)b(h^{-1}g)$ for $g\in G$,
and the adjoint is defined as $a^*(g)=\overline{a(g^{-1})}$ for $g\in G$.
The norm is $\Vert a\Vert=\sup_{[\pi]\in\hat{G}}\Vert\pi(a)\Vert$,
where $\hat{G}$ is the set of equivalence classes of irreducible unitary representations of $G$.
Note that if $G$ is an abelian group, then the product is commutative.
On the other hand, if $G$ is not an abelian group, then the product is noncommutative.
For example, we can use $C^*(G)$ to construct group equivariant neural networks (see Subsection~\ref{subsec:interaction}).
\end{enumerate}
\end{example}

\section{Representing Data and Models Using $C^*$-algebra}\label{sec:data_model}
Recently, machine learning problems
are getting more and more complicated.
As we saw in Section~\ref{sec:c_alg}, $C^*$-algebra is a natural mathematical framework to generalize the notion of complex values to functions and operators.
Thus, applying functions and operators in $C^*$-algebras helps us deal with these complicated situations.
We can effectively represent structured data and multiple models using $C^*$-algebras.
At least, we have the following perspectives.

    \paragraph{Data structure} In many cases, data is not just composed of finite dimensional vectors but composed of time series, graphs, large images, and so on.
    To analyze these kinds of data with higher accuracy, we need to consider the structure of the data and represent it properly.
    $C^*$-algebra helps us represent the data structure.
    For example, if we have finite time-series data with a constant time interval, then we can represent the series with a finite-dimensional vector.
    However, if the series is infinite or if the time interval is not constant, then it is more reasonable to use a function to represent the time series.
    In addition, for graph data, we can use adjacent matrices to represent the graphs.
    Images can also be regarded as functions that map a pixel to the intensity of the pixel.
    Functions and matrices (operators) are perfect tools to represent the rich structure of data.

    \paragraph{Multiple models} In ensemble, multitask, and meta-learning, we consider multiple models simultaneously.
    In these cases, representing the models simultaneously using functions in $C^*$-algebras is more effective than representing each of them individually since we can use tools of \red{functional analysis} to
    %
    %
    extract common features of the models.
    \citet{hashimoto22} used integral and regression to extract common features regarding the gradients of the models. 

    \paragraph{Limited number of samples} In few-shot learning, we try to train models with a limited number of samples.
    We often come across situations where the number of training samples is limited.
    For example, we do not always have enough healthcare data, biological data, abnormal data in anomaly detection, and so on.
    In this case, we need to extract as much information as possible from these samples.
    By using functions in $C^*$-algebras, we can represent infinitely many models, which enables us to extract a maximal amount of features.


{Regarding the data structure, we can also deal with structured data such as functional data with other methods. For example, stochastic processes~\citep{zhu11}, operator learning~\citep{kovachki23}, vector-valued RKHSs~\citep{kadri16}, the framework of functional data analysis~\citep{wang16}. 
Advantages of applying $C^*$-algebras compared to them is summarized as follows.}

\paragraph{Product structure}
{A $C^*$-algebra has the product structure. It enables us to generalize algorithms on Hilbert spaces to those on Hilbert $C^*$-modules. 
Regarding functional data, we can also use other basic function spaces such as $L^2$ spaces and Sobolev spaces that do not have product structures. However, the above generalizations are not possible with them. Similarly, regarding graph data, we can also vectorize a $\mathbb{C}^{d\times d}$ adjacency matrix and regard it as a $d^2$-dimensional vector.
However, we do not have the product structure in that case.
In addition, $C^*$-algebras allow "flexible" product structure. Depending on the $C^*$-algebra, we can use and take advantage of different product structures. For example, the product structure is the convolution for group $C^*$-algebras (See Example~\ref{ex:vonNeumann}.4). We can also apply noncommutative product structures to induce interactions~(See Subsection~\ref{subsec:interaction} for more details).}

\paragraph{Norm}
{The norm in $C^*$-algebras is useful for obtaining theoretical evaluations with milder dependencies on the data dimension than other norms. Indeed, for matrices, we can use the operator norm to alleviate the dependency of the generalization bound on the output data dimension, compared to the case where we regard a $d$ by $d$ matrix as a $d^2$-dimensional vector and use the vector norms such as the Euclidean norm~\citep{hashimoto23-deeprkhm}.}

\paragraph{Inner product}
{We can naturally generalize the notion of inner product and Hilbert space by using $C^*$-algebra, which allows us to generalize RKHS to the space of $C^*$-algebra-valued functions. Learning $C^*$-algebra-valued maps is of great importance in many practical problems where the outputs to be predicted are not scalars but complex and structured data. The generalization of the notion of inner product and Hilbert space is by virtue of the properties of $C^*$-algebra, and this type of generalization is not easy for other notions than $C^*$-algebra. See Section~\ref{sec:hilbert_c_module} and Subsection~\ref{subsec:rkhm} for more details.}

\section{$C^*$-algebra-valued Inner Product: The First Step in Constructing Algorithms}\label{sec:hilbert_c_module}


After representing data and models using $C^*$-algebras, we incorporate them into algorithms.
In the algorithms, we generalize real- or complex-valued notions to $C^*$-algebra-valued ones.
%
For methods implemented in Hilbert spaces, an important notion is the inner product.
For example, a projection of a vector onto a low-dimensional subspace is obtained by computing inner products between the vector and vectors in an orthonormal basis of the subspace.
In addition, for functions in RKHSs~\citep{saitoh16,hashimoto21}, the evaluation of a function at a point is obtained by computing the inner product of the function and the feature vector corresponding to the point~(see Subsection~\ref{subsec:rkhm} for details).
To analyze structured data such as functional and graph data represented by a $C^*$-algebra, generalizing the inner product to the $C^*$-algebra enables us to extract more information than the standard complex-valued inner product.

The space that has the structure of the $C^*$-algebra-valued inner product is called Hilbert $C^*$-module~\citep{lance95}, which is a generalization of Hilbert space.
In the following, we review the definition of Hilbert $C^*$-module.
We first introduce $C^*$-module over a $C^*$-algebra $\alg$, which is a generalization of a vector space.
\begin{definition}[$C^*$-module]\label{def:c*module}
Let $\modu$ be an abelian group with an operation $+$.
If $\modu$ {is equipped with} a (right) $\alg$-multiplication, $\modu$ is called a (right) {\em $C^*$-module} over $\alg$.
\end{definition}
We replace the vector space with a $C^*$-module to represent data.
For example, we use $\mathbb{R}^d$ or $\mathbb{C}^d$ to represent $d$ real- or complex-valued elements of a sample.
If a sample is composed of $d$ elements in $\alg$, e.g., $d$ functions, then we replace $\mathbb{R}^d$ or $\mathbb{C}^d$ with a $C^*$-module $\alg^d$.
%
Although we consider right multiplications in this paper,
considering left multiplications instead of right multiplications is also possible.
%

In a $C^*$-module, we can consider an $\alg$-valued 
inner product.
\begin{definition}[$\alg$-valued inner product]\label{def:innerproduct}
A $\mathbb{C}$-linear map with respect to the second variable $\bracket{\cdot,\cdot}_{\modu}:\modu\times\modu\to\alg$ is called an $\alg$-valued {\em inner product} if it satisfies the following properties for $u,v,w\in\modu$ 
and $c,d\in\alg$:
\begin{enumerate}[leftmargin=*,nosep,topsep=-3pt]
 \item $\bracket{u,vc+wd}_{\modu}=\bracket{u,v}_{\modu}c+\bracket{u,w}_{\modu}d$,
 \item $\bracket{v,u}_{\modu}=\bracket{u,v}_{\modu}^*$,
 \item $\bracket{u,u}_{\modu}\ge_{\alg} 0$,
 \item If $\bracket{u,u}_{\modu}=0$ then $u=0$.
\end{enumerate}
\end{definition}

Analogous to the case of $C^*$-algebra, we can define two notions to measure the magnitude of an element in a $C^*$-module $\modu$ equipped with an $\alg$-valued inner product. 
\begin{definition}[$\alg$-valued absolute value and norm]\label{def:absolute_norm}
For $u\in\modu$, the {\em $\alg$-valued absolute value} $\vert u\vert_{\modu}$ on $\modu$ is defined by the positive element $\vert u \vert_{\modu}$ of $\alg$ such that $\vert u\vert_{\modu}^2=\bracket{u,u}_{\modu}$.   
The (real-valued) norm $\Vert \cdot\Vert_{\modu}$ on $\modu$ is defined by $\Vert u\Vert_{\modu} =\big\Vert\vert u\vert_{\modu}\big\Vert_\alg$. 
\end{definition}
\begin{definition}[Hilbert $C^*$-module]\label{def:hil_c*module}
Let $\modu$ be a $C^*$-module over $\alg$ equipped with an $\alg$-valued inner product.
If $\modu$ is complete with respect to the norm $\Vert \cdot\Vert_{\modu}$, it is called a {\em Hilbert $C^*$-module} over $\alg$ or {\em Hilbert $\alg$-module}.
\end{definition}

\section{$C^*$-algebraic Kernel Methods and Neural Networks}
We present two examples of machine learning methods, kernel methods and neural networks, to show how and why we apply $C^*$-algebra.
\subsection{$C^*$-algebra and kernels: from Hilbert spaces to Hilbert modules}\label{subsec:rkhm}
Reproducing kernel Hilbert spaces (RKHSs) enable us to extract nonlinear features of data~\citep{scholkopf01,saitoh16}.
We first define a complex-valued function $k$, which is called positive definite kernel,
and construct a Hilbert space called RKHS using $k$.
Since RKHSs have high representation power and are theoretically solid, they have been applied to various machine learning methods, such as principal component analysis, support vector machine, and regression.
However, RKHSs and its vector-valued generalization vvRKHSs
are complex- and vector-valued function spaces; the output of the models are the usually complex- and vector-valued, respectively.
Thus, we cannot represent functions whose outputs are in $C^*$-algebras.
In addition, for structured data, complex-valued kernels degenerate the data to a complex value and extracting the information on the structure of data is difficult.
Therefore, the construction of an appropriate positive definite kernel $k$ is not easy.
To resolve these issues, we generalize the positive definite kernel and RKHS by means of $C^*$-algebra~\citep{heo08,hashimoto22}.
Then, we can define reproducing kernel Hilbert $C^*$-module (RKHM), which is a generalization of RKHS.
RKHMs are Hilbert $C^*$-modules.
Thus, they have $C^*$-algebra-valued inner products.
In addition, they are spaces of $C^*$-algebra-valued functions.
By setting a suitable $C^*$-algebra, we can design a suitable RKHM for the given data.
Figure~\ref{fig:rkhm} shows an overview of kernel methods with RKHMs.
Applying RKHMs gives us a new twist on kernel methods.

\begin{figure}[t]
    \centering
    \includegraphics[scale=0.35]{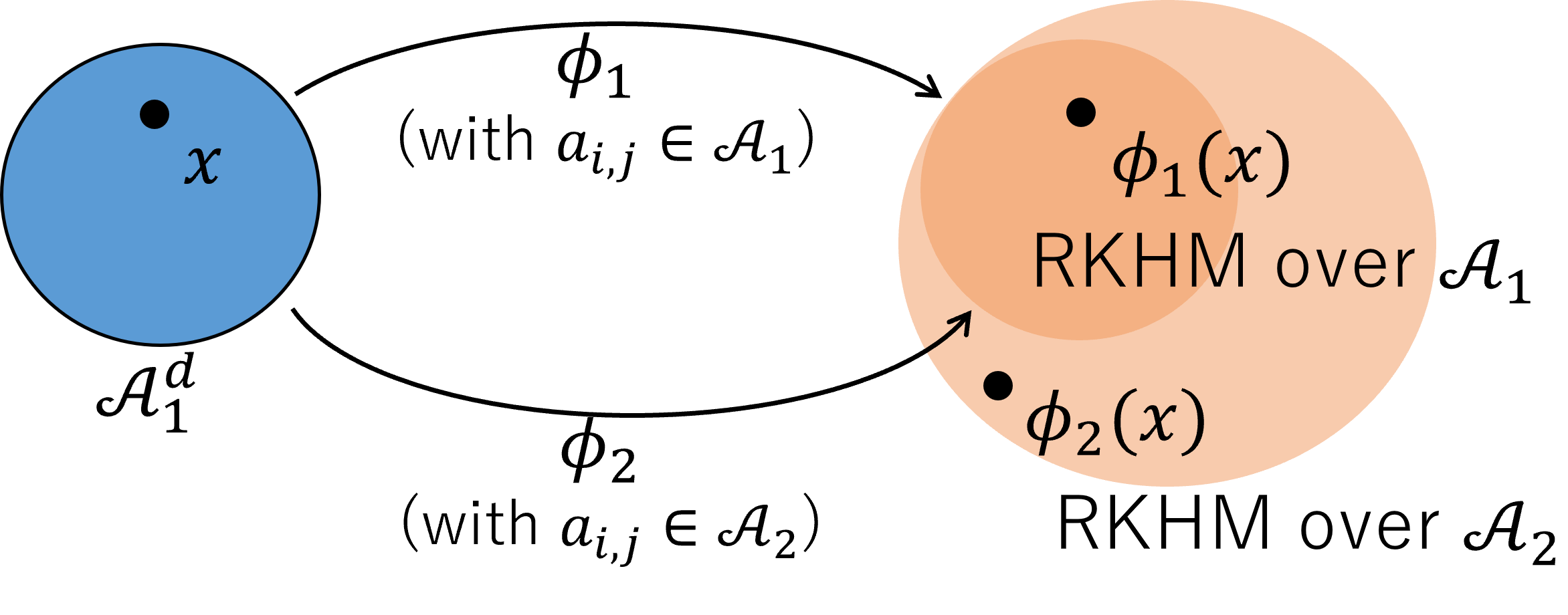}\vspace{-.5cm}
    \caption{Overview of kernel methods with RKHMs by \citet{hashimoto23-aistats}. Here, $\alg_1$ and $\alg_2$ are $C^*$-algebras and $a_{i,j}$ is the parameter of the $C^*$-algebra-valued positive definite kernel associated with the feature maps $\phi_1$ and $\phi_2$. If $\alg_1\subseteq \alg_2$, then the RKHM over $\alg_1$ is contained in the RKHM over $\alg_2$.}\vspace{-.3cm}
    \label{fig:rkhm}
\end{figure}

We review the definition of RKHM below.
First, we define an $\alg$-valued positive definite kernel on a set $\mcl{X}$ for data.
\begin{definition}[$\alg$-valued positive definite kernel]\label{def:pdk_rkhm}
 An $\mcl{A}$-valued map $k:\ \mcl{X}\times \mcl{X}\to\mcl{A}$ is called a {\em positive definite kernel} if it satisfies the following conditions: 
\begin{enumerate}[leftmargin=*,nosep]
 \item $k(x,y)=k(y,x)^*$ \;for $x,y\in\mcl{X}$,
 \item $\sum_{i,j=1}^nc_i^*k(x_i,x_j)c_j\!\!\ge_{\alg} 0$ \;for $n\in\mathbb{N}$, $c_i\in\alg$, $x_i\in\mcl{X}$.
\end{enumerate}
\end{definition}
\citet{hashimoto23-aistats} proposed to constructing $C^*$-algebra-valued kernels using the product structure of the $C^*$-algebra.
They considered a kernel with circulant matrices and the product with matrix-valued parameters of the kernel, which enables us to use an operation that goes beyond the convolution.

Let $\phi:\mcl{X}\to\alg^{\mcl{X}}$ be the {\em feature map} associated with $k$, which is defined as $\phi(x)=k(\cdot,x)$ for $x\in\mcl{X}$.
We construct the following $C^*$-module composed of $\alg$-valued functions:
\begin{equation*}
\modu_{k,0}:=\bigg\{\sum_{i=1}^{n}\phi(x_i)c_i\bigg|\ n\in\mathbb{N},\ c_i\in\alg,\ x_i\in\mcl{X}\bigg\}.
\end{equation*}
Let $\bracket{\cdot,\cdot}_{\modu_k}:\modu_{k,0}\times \modu_{k,0}\to\alg$ defined as
\begin{equation*}
\Bbracket{\sum_{i=1}^{n}\phi(x_i)c_i,\sum_{j=1}^{l}\phi(y_j)d_j}_{\modu_k}:=\sum_{i=1}^{n}\sum_{j=1}^{l}c_i^*k(x_i,y_j)d_j
\end{equation*}
for $c_i,d_i\in\alg$ and $x_i,y_i\in\mcl{X}$.
By the properties in Definition~\ref{def:pdk_rkhm} of $k$, $\bracket{\cdot,\cdot}_{\modu_k}$ is an $\alg$-valued inner product and has the reproducing property
\begin{equation*}
\bracket{\phi(x),v}_{\modu_k}=v(x)
\end{equation*}
for $v\in\modu_{k,0}$ and $x\in\mcl{X}$.
Since $v$ is an $\alg$-valued function, this reproducing property enables us to deal with $\alg$-valued functions, such as the regression of $\alg$-valued functions~\citep{hashimoto23-aistats}.

The {\em reproducing kernel Hilbert $\alg$-module (RKHM)} associated with $k$ is defined as the completion of $\modu_{k,0}$.  
We denote by $\modu_k$ the RKHM associated with $k$.  

An advantage of using $C^*$-algebras is that the operator norm is available.
For the case of $\alg=\mathbb{C}^{d\times d}$, we can also regard $\alg$ as a Hilbert space equipped with the Hilbert--Schmidt inner product.
However, in this case, the norm of a matrix $a\in\mathbb{C}^{d\times d}$ is calculated as $\sum_{i=1}^d\sum_{j=1}^d\vert a_{i,j}\vert^2$, where $a_{i,j}$ is the $(i,j)$-entry of $a$.
On the other hand, the operator norm of $a$ is calculated as $\max_{\Vert v\Vert=1}\sum_{i=1}^d\vert\sum_{j=1}^d a_{i,j}v_j\vert^2$.
Since $\vert \sum_{j=1}^d a_{i,j}v_j\vert\le \sum_{j=1}^d\vert a_{i,j}\vert$ and $\sum_{i=1}^d\vert v_j\vert^2=1$, the dependency of the operator norm on the dimension $d$ is milder than that of the Hilbert--Schmidt norm.
This fact is useful for deriving the generalization bound of the kernel ridge regression.
By virtue of introducing $C^*$-algebra and using the operator norm, we can alleviate the dependency of the generalization bound on the output dimension~\citep{hashimoto23-deeprkhm}.

\subsubsection{Kernel mean embedding}\label{subsec:kme}
Kernel mean embedding enables us to generalize kernel methods to analyze the distribution of data~\cite{kernelmean,sriperumbudur11}.
We define a map that maps a distribution to a vector in an RKHS by integrating the positive definite kernel with respect to the distribution.
This map is called kernel mean embedding and enables us to analyze the distribution in the RKHS.
We can generalize the kernel mean embedding using $C^*$-algebras~\citep{hashimoto22}.
Theoretically, to define the kernel mean embedding, we need the Riesz representation theorem.
Although the Riesz representation theorem is always true for Hilbert spaces, we do not always have the corresponding theorem for Hilbert $C^*$-modules.
According to \citet{skeide00}, we have the Riesz representation theorem if the Hilbert $C^*$-module is in a special class called von-Neumann module (see Definition 4.4 in~\citealt{skeide00}).
In this case, instead of the $[0,1]$-valued (more generally, finite signed or complex-valued) measure, we can map an $\alg$-valued measure~\citep{hashimoto21} to a vector in an RKHM.
$\alg$-valued measures are defined as the special case of vector measures~\citep{dinculeanu67,dinculeanu00}.
Spectral measures and positive operator-valued measures are examples of $\mcl{B}(\mcl{W})$-valued measures for some Hilbert space $\mcl{W}$.
Positive operator-valued measures are introduced in quantum mechanics and
are used in extracting information on the probabilities of outcomes from a state~\citep{holevo11}.
Using the kernel mean embedding for $\alg$-valued measures, we can analyze positive operator-valued measures.
{\citet{hashimoto21} used the principal component analysis with kernel mean embedding to RKHM to analyze the strength of interaction effects for functional data. They also propose a MMD (maximal mean discrepancy) with kernel mean embedding to RKHM. Using the MMD with RKHM, we can define a distance between two positive operator-valued measures.
We can also apply it to anomaly detection for quantum states~\citep{hashimoto_von_neumann}.}
%

\subsubsection{Deep learning with kernels}
Combining kernel methods and deep learning to take advantage of the representation power and the theoretical solidness of kernel methods, and the flexibility of deep learning has been investigated~\cite {cho09,gholami16,bohn19,laforgue19}.
A generalization of these methods to RKHMs is also proposed~\citep{hashimoto23-deeprkhm}.
In this method, instead of considering the composition of functions in RKHSs, we consider the composition of functions in RKHMs.
The high representation power of RKHMs and the product structure of $C^*$-algebras make the deep learning method with kernels more powerful. 

\subsection{Neural network parameters}\label{subsec:c_star_net}
In classical neural networks, the input and output are vectors whose elements are real or complex values. 
The learning parameters are also real or complex values.
If the input and output are represented using $C^*$-algebras, 
\red{then in some cases the corresponding parameters should also be $C^*$-algebra-valued.}
\citet{hashimoto22} proposed $C^*$-algebra net to combine multiple neural network models into a neural network with $C^*$-algebra-valued parameters.
Figure~\ref{fig:c_star_net} shows an overview of the $C^*$-algebra net.
Using this new framework, we can combine multiple neural networks continuously, which is expected to be effective for ensemble, multitask, and meta-learning to fully extract features of data from multiple models or tasks.
In addition, the experiment by \citet{hashimoto22} shows that this framework is useful for the case where the number of training samples is limited.
As we stated in Section~\ref{sec:data_model}, we often come across these situations.

\begin{figure}[t]
    \centering
    \includegraphics[scale=0.23]{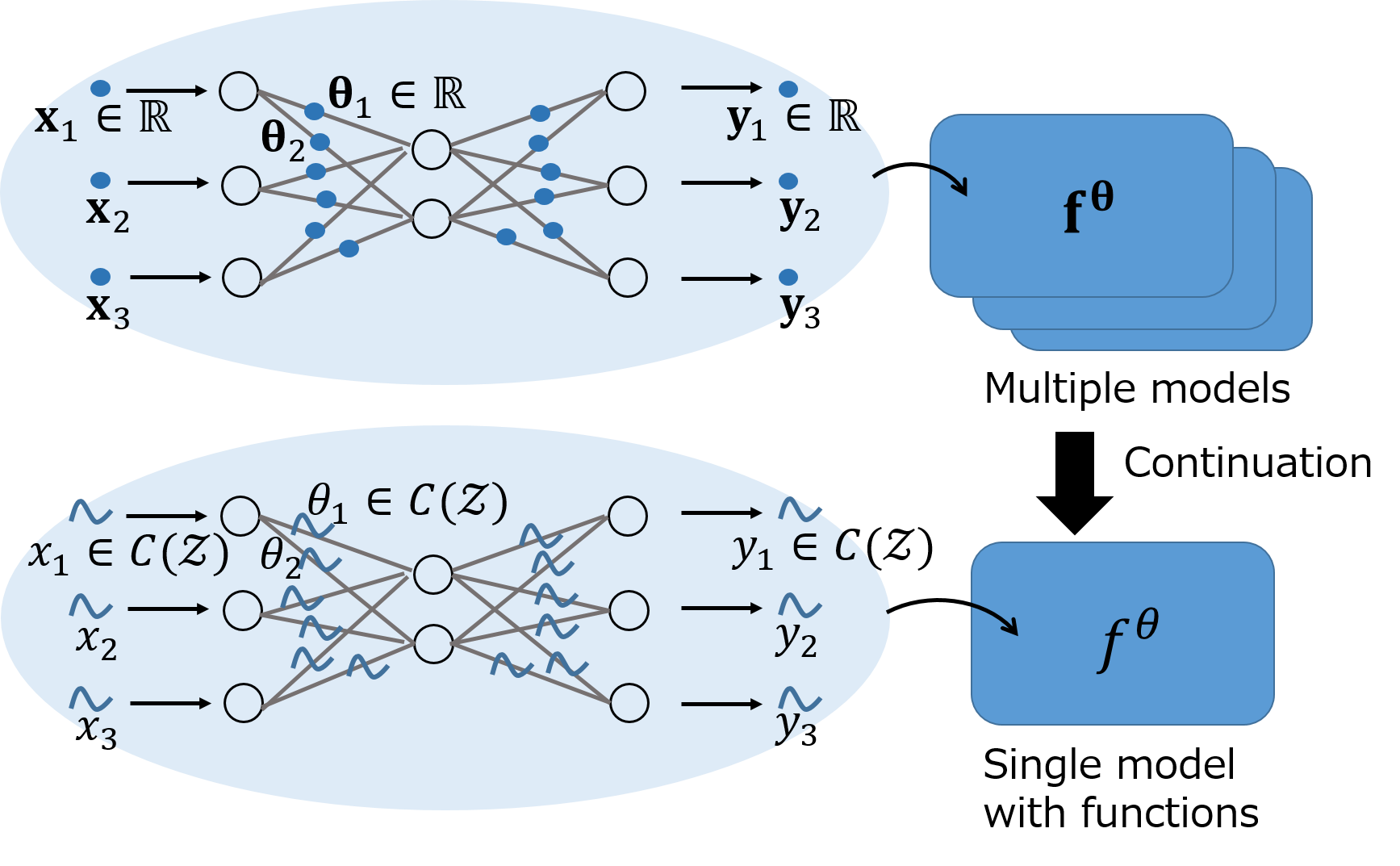}\vspace{-.5cm}
    \caption{Overview of $C^*$-algebra net by \citet{hashimoto22}. They focused on the $C^*$-algebra $C(\mcl{Z})$ for a compact Hausdorff space $\mcl{Z}$ and generalized neural network parameters to $C^*$-algebra-valued. We can continuously combine multiple (real-valued) neural networks using a single $C^*$-algebra net.}\vspace{-.3cm}
    \label{fig:c_star_net}
\end{figure}

We review the technical details of $C^*$-algebra net.
In this subsection, we focus on the case where $\alg$ is the $C^*$-algebra $C(\mcl{Z})$ for a compact Hausdorff space $\mcl{Z}$.
Let $L$ be the number of layers, and for $j=0,\ldots,L$, let $d_j$ be the width of the $j$th layer ($d_0$ is the input dimension).
Let $W^j\in\alg^{d_j\times d_{j-1}}$ and $b^j\in\alg^{d_j}$ be the weight matrix and the bias of the $j$th layer, each of whose element is in $\alg$.
In addition, let $\sigma_j:\alg^{d_j}\to\alg^{d_j}$ be a nonlinear activation function.
The $C^*$-algebra net $f$ is defined as
\begin{align*}
f(x)=\sigma_L(W^L\sigma_{L-1}( \cdots \sigma_1(W^1x+b^1)+\cdots)+b^L)
\end{align*}
for $x\in\alg^{d_0}$.
If $\sigma_j$ is pointwise, i.e., $\sigma_j(x)(z)=\tilde{\sigma}_j(x(z))$ for any $x\in\alg^{d_j}$ and some $\tilde{\sigma}_j:\mathbb{C}^{d_j}\to\mathbb{C}^{d_j}$, then we have
\begin{align}
f(x)(z)=
\tilde{\sigma}_L(W^L(z)\tilde{\sigma}_{L-1}(\cdots\tilde{\sigma}_1(W^1(z)&x(z)\nn\\
+b^1(z))\cdots)+b^L(z)).\label{eq:c_star_net}
\end{align}
Thus, we can represent infinitely many neural networks $\{f(x)(z)\}_{z\in\mcl{Z}}$ by using a single $C^*$-algebra net.
By learning the $\alg$-valued parameters, we can learn the parameters of infinitely many neural networks simultaneously.
In the following, we denote $f(x)(z)$ by $f_z(x)$.

A $C^*$-algebra net $f$ provides infinitely many $\mathbb{C}^{d_L}$-valued networks $f_z$ indexed by $z$.
If we need a single $\mathbb{C}^{d_L}$-valued network, we can integrate $f$ over $\mcl{Z}$.
Assume $\mcl{Z}$ is a measurable space and for any $x\in\alg^{d_0}$, $z\mapsto f_z(x)$ is measurable.
Let $P$ be a probability measure on $\mcl{Z}$.
Consider a map $\tilde{f}:\alg^{d_0}\to\mathbb{C}^{d_L}$ defined as
\begin{align}
\tilde{f}(x)=\int_{\mcl{Z}}f_z(x)\mr{d}P(z),\label{eq:c_star_net_integral}
\end{align}
which is an ensemble of functions $\{f_z\}_{z\in\mcl{Z}}$ with respect to the probability measure $P$.
 
In practical computations, we cannot deal with the infinite-dimensional space $\alg$ itself.
Thus, we restrict $\alg$ to a finite-dimensional space.
Let $\{v_1,\ldots ,v_m\}$ be a basis of the finite-dimensional space.
We represent each element of the weights as $w_{i,k}^j=\sum_{l=1}^mc_{l,i,k}^jv_l$ with coefficients $c_{l,i,k}^j\in\mathbb{C}$.
Here, $w_{i,k}^j$ is the $(i,k)$-element of the $\alg^{d_j\times d_{j-1}}$-valued weight matrix $W^j$.
Then, the $j$th layer $f_j$ of the $C^*$-algebra net is represented as 
\begin{align}
(f_j(x))_i&
=\sigma_j\bigg(\sum_{k=1}^{d_{j-1}}\sum_{l=1}^mc_{l,i,k}^jv_lx_k^{j-1}+b^j_i\bigg)\in\alg,\label{eq:discretized_c_star_net}
\end{align}
where $(f_j(x))_i$ is the $i$th element of the $\alg^{d_j}$-valued vector $f_j(x)\in\alg^{d_j}$.
In addition, $x^0=x$ is the input and $x^j=\sigma_{j}(Wx^{j-1}+b^j)$ is the output of $j$th layer.
In this case, the weight parameters of the $j$th layer of the $C^*$-algebra net are described by $md_jd_{j-1}$ real- or complex-valued parameters.

In the following, we discuss advantages of $C^*$-algebra net. We first show new results about expressiveness and optimization.
Then, we discuss existing investigations about interactions among models.

\subsubsection{Expressiveness}\label{subsec:expressiveness}
An advantage of the $C^*$-algebra net is the expressiveness with respect to the variable $z$.
We focus on a simple case for the discretized version of the $C^*$-algebra net~\eqref{eq:discretized_c_star_net} and show that the representation power of $f_z(x)$ grows as $L$ grows even in this simple case.
We will see that the $C^*$-algebra net is a polynomial of $v_1(z),\ldots,v_l(z)$. 
The $C^*$-algebra net depends on $z$ and $x$ in different ways.
It can be useful for the case where $z$ and $x$ have different attributions.
For example, $z$ is a time variable, and $x$ is a space variable.

Assume the activation function $\sigma_j$ is linear, that is, it satisfies $\sigma_j(\sum_{i=1}^mc_iu_i+b)=\sum_{i=1}^m\sigma_j(c_i)u_i+\sigma_j(b)$ for any $c_1,\ldots,c_m,u_1,\ldots,u_m,b\in\mathbb{C}$.
In addition, for simplicity, we assume the input and the biases are constant functions, i.e., $x(z)=\hat{x}$ and $b^j(z)=\hat{b}^j$ for any $z\in\mcl{Z}$, some $\hat{x}\in\mathbb{C}^{d_0}$, and some $\hat{b}^j\in\mathbb{C}^{d_j}$.
We have the following proposition.

\begin{proposition}\label{prop:c_star_net_poly}
The $C^*$-algebra net $f_z(x)$ is a degree $L$ polynomial with respect to $v_1(z),\ldots,v_m(z)$. 
\end{proposition}

See Appendix~\ref{ap:proofs} for the proof of Proposition~\ref{prop:c_star_net_poly}.
Proposition~\ref{prop:c_star_net_poly} shows that even if this simple case of the activation function $\sigma_j$ is linear, $f_z(x)$ is nonlinear with respect to $z$.
We can also construct a network whose input space is $\mathbb{C}^{d_0}\times \mcl{Z}$.
However, the situation of the $C^*$-algebra net with a finite-dimensional approximation is totally different from the case of the network on $\mathbb{C}^{d_0}\times \mcl{Z}$.
For the case of the network on $\mathbb{C}^{d_0}\times \mcl{Z}$, if all the activation functions are the ReLU, defined as $\sigma(x)=\max\{0,x\}$ for $x\in\mathbb{R}$, then the approximation is obtained by the piecewise linear functions with respect to $z$~\citep{hanin19}.
On the other hand, in the case of the $C^*$-algebra net, if all the activation functions are linear, then the approximation is obtained by the polynomials with respect to $v_l(z)$.
This fact means that even if the activation functions are linear, the representation power of the network with respect to $z$ grows as $L$ becomes large.
{In summary, the expressive power of $C^*$-algebra nets is high enough so that they can represent polynomials with respect to $z$ even if the activation functions are linear. In addition, using $C^*$-algebra nets, we can induce a new type of nonlinearity that is different from the nonlinearity induced by the classical neural networks.}

\subsubsection{Optimization}\label{subsec:optimization}
Another advantage of the $C^*$-algebra net is that it fills the gap between convex and nonconvex optimization problems.
Since the weight matrices of $C^*$-algebra nets are functions, they correspond to infinitely many weight matrices.
Therefore, if we set a $C^*$-algebra appropriately, then we can represent an arbitrary scalar-valued network as a single $C^*$-algebra network.
As a result, the optimization problem of the standard scalar-valued network is reduced to a convex optimization problem of a $C^*$-algebra network.
Convex optimization of neural networks has been proposed and investigated~\citep{bengio05,nitanda17,chizat18,daneshmand23a}.
In these studies, they consider learning the distribution of the weight parameters, which makes the optimization problem convex.
On the other hand, the objective function becomes highly nonconvex if we consider optimizing weight parameters themselves, not the distribution of them.
We show that the framework of $C^*$-algebra nets fills the gap between these convex and nonconvex optimizations.

For simplicity, we focus on neural networks with real-valued parameters and assume the input $x$ is in the form $x(z)=\hat{x}$ for some $\hat{x}\in\mathbb{R}^{d_0}$.
Let $\Omega$ be an interval in $\mathbb{R}$, and let $\mcl{W}$ be a compact space.
Let $\alpha^j_{i,k}:\mcl{W}\to\Omega$ be a surjective function for $j=1,\ldots,L$, $i=1,\ldots,d_j$, and $k=1,\ldots,d_{j-1}$.
The simplest example is setting $\mcl{W}=\Omega$ and $\alpha^j_{i,k}$ as the identity.
Let $N=\sum_{j=1}^L{d_j(d_{j-1}+1)}$, the number of parameters (weight and bias parameters), $\mcl{Z}=\mcl{W}^N$, and $\alg=C(\mcl{Z})$.
Define ${W}^j:\mcl{Z}\to\mathbb{R}^{d_j\times d_{j-1}}$ and ${b}^j:\mcl{Z}^N\to\r{d_j}$ as
\begin{align*}
&w^j_{i,k}(z^1_{1,1},\ldots,z^j_{i-1,k},z,z^j_{i+1,k},\ldots,z^L_{d_{L},d_{L-1}+1})
=\alpha^j_{i,k}(z)\\
&{b}^j_i(z^1_{1,1},\ldots,z^j_{i-1,d_{j-1}+1},z,z^j_{i+1,d_{j-1}+1},\ldots,z^L_{d_{L},d_{L-1}+1})\\
&\qquad=\alpha^j_{i,k}(z)
\end{align*}
for any $z^1_{1,1},\ldots,z^j_{i,k-1},z^j_{i,k+1},\ldots,z^L_{d_{L},d_{L-1}+1}\in\mcl{Z}$
and $z^1_{1,1},\ldots,z^j_{i-1,d_{j-1}+1},z^j_{i+1,d_{j-1}+1},\ldots,z^L_{d_{L},d_{L-1}+1}$, respectively.
Let $\hat{\sigma}:\r{d_j}\to\r{d_j}$ be an activation function for standard scalar-valued networks.
Set the activation function ${\sigma}_j:\alg^{d_j}\to \alg^{d_j}$ as ${\sigma}_j(x)(z)=\hat{\sigma}_j(x(z))$ for any $x\in\alg^{d_j}$ and $z\in\mcl{Z}$.
With the above $W^j$, $b^j$, and $\sigma_j$, we construct the $C^*$-algebra net $f$.
Then, we can show that we can represent any real-valued neural network by the form of $f_z$.

Let $\mcl{F}$ be the class of the $C^*$-algebra nets defined above. 
Assume for any $\hat{x}\in\r{d_0}$, $f(\hat{x})$ is bounded on $\mcl{Z}$.
Let $\mcl{P}(\mcl{Z})$ be the set of probability measures on $\mcl{Z}$.
As we mentioned as Eq.~\eqref{eq:c_star_net_integral}, we integrate ${f}\in\mcl{F}$ over $\mcl{Z}$ with respect to a probability measure $P\in\mcl{P}(\mcl{Z})$ to get a scalar-valued function. 
For $P\in\mcl{P}(\mcl{Z})$, let $A_P$ be the integral operator on $\mcl{F}$ defined as
$A_P f(\hat{x})=\int_{\mcl{Z}}f_z(\hat{x})\mr{d}P(z)$
for $\hat{x}\in\r{d_0}$.
The averaged $C^*$-algebra net $A_Pf(\hat{x})$ is regarded as a continuation of the $(L+1)$-layer neural network $\sum_{i=1}^{d_{L+1}}p_if_{z_i}(\hat{x})$, where $d_{L+1}\in\mathbb{N}$ and $p_i$ is the weight parameter of the final layer.
Figure~\ref{fig:opt} schematically shows the averaged $C^*$-algebra net $A_Pf(\hat{x})$.
We can see that the optimization problem of learning $P$ is convex.
See Appendix~\ref{ap:optimization} for more details.

The cases 1) Fixing $f\in\mcl{F}$ and optimizing $P$, and 2) optimizing a real-valued network $f_z$ with respect to the weight parameters are two extremes.
By virtue of the $C^*$-algebra net, we can define intermediate cases.
Let $K\in\{0,\ldots,N\}$ and let $M_1,\ldots,M_K$ be disjoint subsets of $\mcl{I}:=\{(j,i,k)\,\mid\,j=1,\ldots,L,\ i=1,\ldots,d_j,\ k=1,\ldots,d_{j-1}\}$ and let $\mcl{Z}=\mcl{W}^K$.
If $K=0$, then we set $\mcl{Z}=\mcl{W}^0$ as a singleton, and we also set $M_0$ as an infinite set containing $\mcl{I}$.
Note that in the case where $\mcl{Z}$ is a singleton, $C(\mcl{Z})$ is isomorphic to $\mathbb{C}$.
For $l=1,\ldots,K$, the variables $z^j_{i,k}$ for $(j,i,k)\in M_l$ are tied together and represented as a single variable.
Let
\begin{align*}
w_{i,k}^j(z_1,\ldots,z_{l-1},z,z_{l+1},\ldots,z_K)&=\alpha_{i,k}^j(z)
\end{align*}
for $(j,i,k)\in M_l$.
If $K=N$, then we reconstruct the case 1).
If $K=0$, then we have $w^j_{i,k}\in\mathbb{C}$, and $\alpha^j_{i,k}$ is necessarily a constant function.
In addition, since $\mcl{Z}$ is a singleton, $\mcl{P}(\mcl{Z})=\{1\}$.
As a result, we reconstruct the case 2).
As for the learning, if $(j,k,i)\in M_l$ for some $l$ with $\vert M_l\vert\ge 2$, then we learn $\alpha^j_{i,k}$.
If $(j,i,k)\in M_l$ for some $l$ with $\vert M_l\vert=1$, then we fix $\alpha^j_{i,k}$.
We also learn $P$.

\begin{figure}[t]
    \centering
    \includegraphics[scale=0.44]{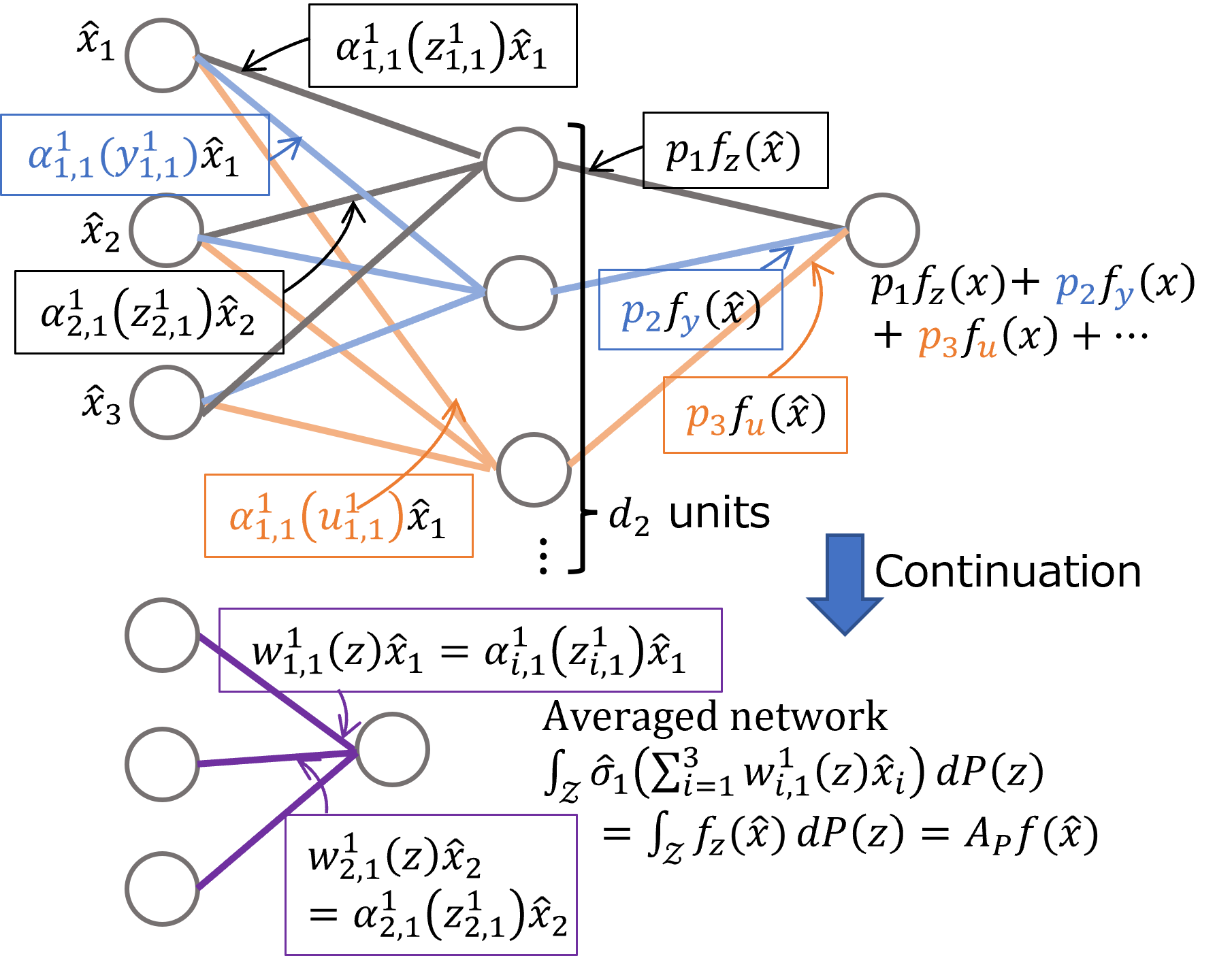}\vspace{-.3cm}
    \caption{Overview of the averaged $C^*$-algebra net $A_Pf(x)$. Here, $f_z(x)=\hat{\sigma}_1(\sum_{i=1}^3\alpha^1_{i,1}(z)x_i)$.
    We can regard $A_Pf(x)$ as a continuation of the $2$-layer neural network $\sum_{i=1}^{d_{2}}p_if_{z_i}(x)$.}
    \label{fig:opt}\vspace{-.5cm}
\end{figure}


\subsubsection{Interactions among models}\label{subsec:interaction}
The product structure of the $C^*$-algebra gives the model additional structures.
A generalization of the $C^*$-algebra net to noncommutative $C^*$-algebra is also proposed~\citep{hataya23_cstar}.
In the above framework of the $C^*$-algebra network, we focus on the $C^*$-algebra of continuous functions, whose product structure is commutative.
Therefore, if we consider the original $C^*$-algebra net~\eqref{eq:c_star_net}, not the discretized one~\eqref{eq:discretized_c_star_net}, then we have a separated neural network $f_z$ for each $z$.
The models do not interact without designing additional regularization terms to the loss function.
By regarding $w\in C(\mcl{Z})$ as the multiplication operator $M_w\in \mcl{B}(L^2(\mcl{Z}))$ defined as $M_wv=w\cdot v$,
we can regard the $C^*$-algebra net over $C(\mcl{Z})$ as that over $\mcl{B}(L^2(\mcl{Z}))$, where $L^2(\mcl{Z})$ is the space of square-integrable functions on $\mcl{Z}$.
For the case where $\mcl{Z}$ is a finite set, $C(\mcl{Z})$ corresponds to the space of squared diagonal matrices, and $\mcl{B}(L^2(\mcl{Z}))$ corresponds to the space of general squared matrices.
Replacing $C(\mcl{Z})$ with $\mcl{B}(L^2(\mcl{Z}))$ and adding the nondiagonal part to the weight parameter $w^j_{i,k}$, 
we cannot separate each network $f_z$ since the product structure becomes more complicated.
\citet{hataya23_cstar} also proposed $C^*$-algebra nets over group $C^*$-algebras, which enable us to construct group equivariant neural networks by virtue of the noncommutative product structure in the group $C^*$-algebra.

\section{Future Directions}
As we discussed in the previous sections, generalizing machine learning methods by means of $C^*$-algebra enables us to go beyond the existing methods.
\red{However, there are many challenges involved. We discuss some them.}

\subsection{\red{Challenges}}

\paragraph{Implementation and computational cost}
When we implement the methods with $C^*$-algebras, we have to represent elements in the $C^*$-algebras so that they are suitable for the computation.
If the $C^*$-algebra is an infinite-dimensional space, we have to somehow discretize elements in the $C^*$-algebra.
{\citet{hashimoto21} use Fourier functions to discretize the functions in a $C^*$-algebra. Using kernel ridge regression to represent functions in a $C^*$-algebra is also proposed~\citep{hashimoto22}. However, the effect of these methods on the entire algorithms has not been investigated, and representing elements in $C^*$-algebra properly is an important future direction of research. }
Even if we have an appropriate discretization method, the computational cost becomes expensive if the number of points for the discretization is large.
{For example, for kernel methods, if we represent elements in the $C^*$-algebra as $d$ by $d$ matrices, then the size of the Gram matrix is $nd$ by $nd$, where $n$ is the sample size. The cost for the computation involving the Gram matrix is expensive if $n$ and $d$ are large. 
In addition, for neural networks, if we represent weight parameters as $d$ by $d$ matrices, then the number of learnable parameters is $d^2$ times as large as that for the standard neural network with the same architecture. To alleviate the dependency on $d$, appropriate representations of elements of $C^*$-algebras to reduce computational costs should be investigated. }
{In addition, although source codes are provided by the authors of the papers, as far as we know, no software for $C^*$-algebraic machine learning has been developed so far. The development of software is crucial to familiarize the machine learning community with the concept of $C^*$-algebra.}

\paragraph{Lack of the inverse} 
{An element in a $C^*$-algebra does not always have its inverse. This situation is different from the standard complex- or real-valued case. This difference makes it difficult for us to generalize algorithms and theorems in Hilbert spaces straightforwardly. In fact, in Hilbert $C^*$-modules, once we normalize a vector, we cannot reconstruct the original vector exactly. However, we can obtain a normalized vector reconstructing a vector that is sufficiently close to the original vector \citep[Proposition 3.2]{hashimoto23_moea}. In addition, we only have an approximate version of the representer theorem for RKHMs \citep[Theorem 4.5]{hashimoto23_moea}. We sometimes have to give up constructing the exactly same algorithms and results as those in Hilbert spaces and devote ourselves to investigate how we can approximately obtain the algorithms and results.}

\paragraph{Dealing with infinite-dimensional spaces}
Proving theoretical aspects of applying $C^*$-algebras to machine learning is not always straightforward.
For example, as we mentioned in Subsection~\ref{subsec:kme}, Riesz representation theorem is not always true for Hilbert $C^*$-modules.
In addition, for a submodule of a Hilbert $C^*$-module, its orthogonal complement does not always exist~\citep{manuilov00}.
Since these properties are fundamental for Hilbert spaces and used in proving and guaranteeing the theoretical aspects of machine learning methods, we have to be careful when we try to analyze methods with $C^*$-algebras theoretically.

\paragraph{Designing kernels}
Further investigation for designing positive definite kernels using $C^*$-algebras would be interesting. 
The kernel proposed by~\citet{hashimoto23-aistats}, which we discussed in Subsection~\ref{subsec:rkhm}, is based on the convolution and is suitable for image data. 
Other kernels for other types of data, such as graphs and functions, should be investigated.

\paragraph{Theory of $C^*$-algebra nets}
Investigating generalization property and implicit regularization of the $C^*$-algebra net is an interesting future work. For example, it would be interesting to consider what types of $C^*$-algebra induce generalization or implicit regularization.
In addition, understanding the standard real-valued neural networks and developing new methods regarding them through $C^*$-algebra net would also be interesting.
For example, constructing an optimization method based on the observation in Subsection~\ref{subsec:optimization} to obtain a better solution is future work.

\paragraph{$C^*$-algebraic quantum machine learning}
{We can represent various notions in quantum machine learning using $C^*$-algebra. For example, density matrices are represented by matrices. Analyzing quantum states using RKHMs or $C^*$-algebra nets is an interesting direction of research. In addition, quantum gates are represented by unitary matrices. Constructing or analyzing quantum circuits using $C^*$-algebra nets is also an interesting direction of research. $C^*$-algebras could also give rise to new machine learning methods which can be implemented more efficiently using a quantum computer.}

\subsection{Further examples of $C^*$-algebra for future work}
\paragraph{Cuntz algebra}
Let $\mcl{W}$ be a Hilbert space. 
Cuntz algebra $\mcl{O}_n$ is a $C^*$-algebra generated by the isometries on $\mcl{W}$, i.e., linear operators on $\mcl{W}$ satisfying $S_i^*S_i=1$ ($i=1,\ldots,n$) and $\sum_{i=1}^nS_iS_i^*=1$, where $1$ is the identity map on $\mcl{W}$~\citep{cuntz77}.
We can represent variable-length data, each of whose element is a discrete value.
For example, we can set $S_1,\ldots,S_n$ as the dictionary of words and represent sentence as the product of the words from $S_1,\ldots,S_n$.
We can also set $S_1,\ldots,S_n$ as nodes and represent paths using the product of the nodes from $S_1,\ldots,S_n$.
Moreover, there are existing studies for applying Cuntz algebras to represent the filters in signal processing~\citep{jorgebsen07}.

\paragraph{Approximately finite dimensional $C^*$-algebra}
A $C^*$-algebra is referred to as approximately finite (AF) dimensional if it is the closure of an increasing union of finite dimensional subalgebras~\citep{davidson96}.
We can use AF $C^*$-algebras for representing data whose dimensions can vary, such as the adjacent matrices of social network graphs.

\section{Conclusion}
We proposed $C^*$-algebraic machine learning and discussed advantages and challenges of applying $C^*$-algebra to machine learning methods.
We can represent structured data and multiple models using $C^*$-algebras, and by incorporating them into algorithms, we can fully extract features of data. 
$C^*$-algebra gives us a new twist on machine learning.
We hope that our analysis will lead to greater attention to $C^*$-algebra in machine learning.

\section*{Acknowledgements}
Hachem Kadri is partially supported by grant ANR-19-CE23-0011 from the French National Research Agency.
Masahiro Ikeda is partially supported by grant JPMJCR1913 from JST CREST.


\bibliography{example_paper}
\bibliographystyle{icml2024}

\newpage
\appendix
\onecolumn
\section*{Appendix}

\section{Details of Subsection~\ref{subsec:optimization}}\label{ap:optimization}
We provide the details of Subsection~\ref{subsec:optimization}.
With $W^j$, $b^j$, and $\sigma_j$ defined in Subsection~\ref{subsec:optimization}, let
\begin{align}
{f}_j(x)&={\sigma}_j({W}^jx+{b}^j)\ (j=1,\ldots,L).\label{eq:tilde_f}
\end{align}
We consider the set of $C^*$-algebra nets defined as $\mcl{F}_{\mcl{Z}}=\{{f}_z\mid {f=f_L\circ\cdots\circ f_1}\mbox{ with }f_j \mbox{in Eq. \eqref{eq:tilde_f}, }z\in\mcl{Z}\}$.
In addition, we set the following function class of the standard real-valued networks:
\begin{align*}
\hat{\mcl{F}}&=\{\hat{f}_L\circ\cdots\circ \hat{f}_1\,\mid\,\hat{f}_j(x)=\hat{\sigma}_j(\hat{W}^jx+\hat{b}^j),\ 
\hat{W}^j\in\Omega^{d_j\times d_{j-1}},\ \hat{b}^j\in\Omega^{d_j}\}.
\end{align*}

\begin{proposition}\label{prop:function_class_equiv}
We have $\mcl{F}_{\mcl{Z}}=\hat{\mcl{F}}$ as sets.
\end{proposition}

\begin{proof}
Let ${f}_z\in \mcl{F}_{\mcl{Z}}$ for some $z\in\mcl{Z}$.
Let $\hat{w}^j_{i,k}={w}^j_{i,k}(z)$ and $\hat{b}^j_i={b}^j_i(z)$ for any $j=1,\ldots,L$, $i=1,\ldots,d_j$, and $k=1,\ldots,d_{j-1}$.
In addition, let $\hat{\sigma}_j(x)=\sigma_j(x1_{\alg})(z)$ for $x\in\r{d_j}$, where $1_{\alg}$ is the constant map defined as $1_{\alg}(z)=1$ for any $z\in\mcl{Z}$.
We construct $\hat{f}=\hat{f}_L\circ \cdots\circ \hat{f}_1$ as $\hat{f}_j(x)=\hat{\sigma}_j(\hat{W}^jx+\hat{b}^j)$ with $\hat{W}^j$, $\hat{b}^j$, and $\hat{\sigma}_j$ defined above.
Then, we have $f_z=\hat{f}$ and $f_z\in\hat{\mcl{F}}$.
The converse is trivial by the definition of $\mcl{F}_{\mcl{Z}}$.
\end{proof}

Let $\mcl{F}$ be the class of the functions $f=f_L\circ \cdots\circ f_1$ defined as Eq.~\eqref{eq:tilde_f}.
Assume for any $\hat{x}\in\r{d_0}$, $f(\hat{x})$ is bounded on $\mcl{Z}$.
Let $\mcl{P}(\mcl{Z})$ be the set of probability measures on $\mcl{Z}$.
As we mentioned as Eq.~\eqref{eq:c_star_net_integral}, we integrate ${f}\in\mcl{F}$ over $\mcl{Z}$ with respect to a probability measure $P\in\mcl{P}(\mcl{Z})$ to get a scalar-valued function. 
For $P\in\mcl{P}(\mcl{Z})$, let $A_P$ be the integral operator on $\mcl{F}$ defined as
$A_P f(\hat{x})=\int_{\mcl{Z}}f_z(\hat{x})\mr{d}P(z)$
for $\hat{x}\in\r{d_0}$.
The following proposition shows the optimization problem of learning $P$ is convex.

\begin{proposition}\label{prop:convex}
Let $\mcl{L}:\r{d_L}\times\r{d_L}\to\mathbb{R}_+$ be a loss function that is continuous and where for any $y\in\r{d_L}$, the function $\mcl{L}(\cdot,y)$ on $\r{d_L}$ is convex.
Fix $f$ as a network defined as \eqref{eq:tilde_f}.
Then, for any $\hat{x}\in\r{d_0}$ and any $y\in\r{d_L}$, the map $\mcl{P}(\mcl{Z})\to \mathbb{R}_+$ defined as $P\mapsto \mcl{L}(A_Pf(\hat{x}),y)$ is convex.
\end{proposition}

\begin{proof}
Since $\mcl{L}(\cdot,y)$ is convex, for $P,Q\in\mcl{P}(\mcl{Z})$ and $t\in[0,1]$, we have
\begin{align*}
\mcl{L}(A_{tP+(1-t)Q}f(x),y)&=\mcl{L}(tA_P f(x)+(1-t)A_Q f(x),y)\\
&\le t\mcl{L}(A_Pf(x),y)+(1-t)\mcl{L}(A_Q f(x),y),
\end{align*}
which completes the proof of the proposition.
\end{proof}

\begin{remark}\label{rem:dirac}
Let $P=\delta_z$, where $\delta_z$ is the Dirac measure centered at $z\in\mcl{Z}$.
Then, $A_{\delta_z}f=f_z$.
Proposition~\ref{prop:function_class_equiv} implies that learning $z$ of $f_z$ for a $C^*$-algebra network $f$ corresponds to learning the weights $\hat{W}^j$ and biases $\hat{b}^j$ of the scalar-valued network $\hat{f}$.
Note that the set $\{A_{\delta_z}\,\mid\,z\in\mcl{Z}\}$ is not convex.
By expanding the search space to the set of probability measures, the optimization problem becomes convex.
\end{remark}

\section{Proof of Proposition~\ref{prop:c_star_net_poly}}\label{ap:proofs}
\begin{mythm}[Proposition~\ref{prop:c_star_net_poly}]
The $C^*$-algebra net $f_z(x)$ is a degree $L$ polynomial with respect to $v_1(z),\ldots,v_m(z)$.
\end{mythm}

\begin{proof}
The $k_L$th element of the output of $f_z(x)$ is written as
\begin{align*}
&(f_z(x))_{k_L}\nn\\
&=\sigma_L\bigg(\sum_{k_{L-1}=1}^{d_{L-1}}\sum_{l_L=1}^mc^L_{l_L,k_L,k_{L-1}}v_{l_L}(z)\sigma_{L-1}\bigg(\cdots
\sigma_2\bigg(\sum_{k_1=1}^{d_1}\sum_{l_2=1}^mc^2_{l_2,k_2,k_1}v_{l_2}(z)\sigma_1\bigg(\sum_{k_0=1}^{d_0}c^1_{l_1,k_1,k_0}v_{l_1}(z)\hat{x}_{k_0}\nn\\
&\quad +\hat{b}^1_{k_1}\bigg)+\hat{b}^2_{k_2}\bigg)\cdots\bigg)+ \hat{b}_{k_L}^L\bigg)\nn\\
&=\!\!\!\!\!\sum_{l_1,\ldots,l_L=1}^mv_{l_L}(z)\cdots v_{l_1}(z)\sigma_L\bigg(\sum_{k_{L-1}=1}^{d_{L-1}}c^L_{l_L,k_L,k_{L-1}}\sigma_{L-1}\bigg(
\cdots\sigma_1\bigg(\sum_{k_0=1}^{d_0}c^1_{l_1,k_1,k_0}\hat{x}_{k_0}\bigg)\cdots\bigg)\nn\\
&\quad+\!\!\!\!\!\!\!\sum_{l_2,\ldots,l_L=1}^mv_{l_L}(z)\cdots v_{l_2}(z)\sigma_L\bigg(\sum_{k_{L-1}=1}^{d_{L-1}}c^L_{l_L,k_L,k_{L-1}}\sigma_{L-1}\bigg(
\cdots\sigma_2\bigg(\sum_{k_1=1}^{d_1}c^2_{l_2,k_2,k_1}\sigma_1(\hat{b}_{k_1}^1)\bigg)\cdots\bigg)\nn\\ 
&\quad+\sum_{l_L=1}^mv_{l_L}(z)\sigma_L\bigg(
\sum_{k_{L-1}=1}^{d_{L-1}}c^L_{l_L,k_L,k_{L-1}}\sigma_{L-1}(\hat{b}^{L-1}_{k_{L-1}})\bigg)
+\sigma_L(\hat{b}^L_{k_L}),
\end{align*}
which completes the proof of the proposition.
\end{proof}



\end{document}